\documentclass[11pt]{article}
\usepackage{acl}
\usepackage[T1]{fontenc}
\usepackage[utf8]{inputenc}
\usepackage{times}
\usepackage{latexsym}
\usepackage{microtype}
\usepackage{makecell}

\usepackage{amsmath,amssymb,amsthm,mathtools}
\usepackage{booktabs}
\usepackage{graphicx}
\usepackage{pgfplots}
\usepackage{tikz}
\usetikzlibrary{arrows.meta,positioning}
\pgfplotsset{compat=1.17}

\newtheorem{theorem}{Theorem}

\newtheorem{proposition}[theorem]{Proposition}
\newtheorem{definition}{Definition}

\newcommand{\PRM}{\mathrm{PRM}}
\newcommand{\X}{\mathcal{X}}
\newcommand{\D}{\mathcal{D}}

\newcommand{\R}{\mathcal{R}}
\newcommand{\ind}{\mathbf{1}}
\DeclareMathOperator*{\supop}{sup}

\title{Quality-Diversity Stress Tests for Process Reward Models:\\
What Archive Coverage Can and Cannot Certify}
\author{
\textbf{Ibne Farabi Shihab}\textsuperscript{1}%
\thanks{Equal contribution.}%
\thanks{Corresponding author: \texttt{ishihab@iastate.edu}.}
\and
\textbf{Fariya Afrin}\textsuperscript{2}\footnotemark[1]
\\[2pt]
\textsuperscript{1}Department of Computer Science, Iowa State University \\
\textsuperscript{2}Department of Computer Science, Kalinga Institute of Industrial Technology \\
\texttt{ishihab@iastate.edu}
}

\begin{document}
\maketitle

\begin{abstract}
Process reward models (PRMs) score intermediate reasoning steps and are widely used for search, ranking, and training, but optimization can exploit these learned proxies by increasing reward while turning correct reasoning into incorrect reasoning. We formulate PRM stress testing as a quality-diversity search problem using MAP-Elites, retaining the most severe correctness-flipping edit in each behavior-space region while separating search coverage from exploit coverage. We characterize what such archives certify: finite-cell repair bounds covered-cell tail risk and average residual severity but cannot bound the worst remaining cell from covered fraction alone; under Lipschitz post-repair loss and metric-cover auditing, the residual is bounded by archive fitting error plus the Lipschitz constant times the covering radius. A controlled landscape validates this certificate and the impossibility of any fraction-only worst-case guarantee. On real PRMs, the search reveals an aggregation-dependent vulnerability in \texttt{Qwen2.5-Math-PRM-7B}: padding yields $44$ strict exploits with maximum gain $0.294$ under mean pooling versus one exploit under minimum readout; a matched syntactic control isolates the mechanism, and an RLHFlow value-head model shows the same qualitative effect with maximum gain $0.005$. A predeclared paired LoRA repair protocol reduces exploit rates from $0.148$ to $0.037$--$0.074$, lowers the worst attack from $0.333$ to $0.177$--$0.212$, improves ranking AUROC without degrading best-of-$4$ accuracy, attributes gains to adversarial fine-tuning rather than archive diversity, and is confirmed by independent unpaired replications ($44\!\rightarrow\!1$, clean-split worst gain $0.0092$, MATH-500 $41\!\rightarrow\!0$, clean ranking $40/40$).
\end{abstract}

\section{Introduction}
\label{sec:intro}

Process reward models have become an important component of modern reasoning systems. Instead of
assigning a single score to a completed answer, a PRM evaluates intermediate steps, providing a
dense signal that can guide inference-time search, rank candidate derivations, or serve as a
training reward \citep{lightman2024lets,uesato2022solving,wang2024mathshepherd}. The appeal is
straightforward: a step-level signal can distinguish two reasoning traces long before their final
answers diverge. The same granularity, however, creates a larger surface on which a learned
evaluator can be optimized. When a PRM rewards features that correlate with correctness on its
training distribution but are not constitutive of correctness, search or reinforcement learning can
amplify those features without improving the reasoning itself.

This failure is the process-level analogue of reward over-optimization in outcome reward models
\citep{gao2023scaling,skalse2022defining,amodei2016concrete}. Consider a correct trace and an edit
that replaces its conclusion with a plausible wrong answer while adding several confident
``verification'' steps. If the added steps receive high local scores, an average over step rewards
can increase even though the edited trace is wrong. Such an example is more than a naturally
misclassified trace. It is a direction in trace space along which the proxy improves as the target
degrades, and therefore one that an optimizer is incentivized to exploit.

Finding one such direction is useful but incomplete. A single-objective attack tends to return many
variants of the same high-scoring strategy, while a hand-written test suite covers only the failure
modes anticipated by its designer. Quality-diversity (QD) search provides a natural alternative.
Rather than retaining only the globally strongest attack, MAP-Elites
\citep{mouret2015illuminating,pugh2016quality} keeps the strongest candidate in each region of a
behavior space. In our setting, the regions are defined by the type and magnitude of a trace edit.
The resulting archive reveals whether a PRM exhibits a single narrow vulnerability or a repertoire
of qualitatively distinct failure modes.

An archive also invites a tempting robustness claim: if attacks have been found and repaired in a
fraction $\rho$ of the cells, perhaps the worst remaining vulnerability should fall like
$1-\rho$. This intuition is false. An archive may cover almost every cell while missing one cell
whose severity is maximal. Coverage fraction can control an average or a tail probability, but it
cannot by itself control a supremum. This distinction matters because an easily reported coverage
statistic can otherwise be mistaken for a worst-case certificate that it does not justify.

We therefore separate three notions that are often conflated. Visitation coverage records where
the search evaluated at least one candidate. Exploit coverage records where it found a verified
correctness-flipping score increase. Robustness coverage records where a post-repair certificate
supplies a valid cellwise upper bound; a finite heuristic search can only approximate this
condition empirically. These notions support fundamentally different robustness claims. A finite
partition yields a deterministic covered-tail certificate and an average-severity bound. A
worst-case certificate requires additional structure, which we express through the covering radius
of the audited examples in a metric over complete attack instances.

\begin{figure}[t]
    \centering
    \includegraphics[width=\linewidth]{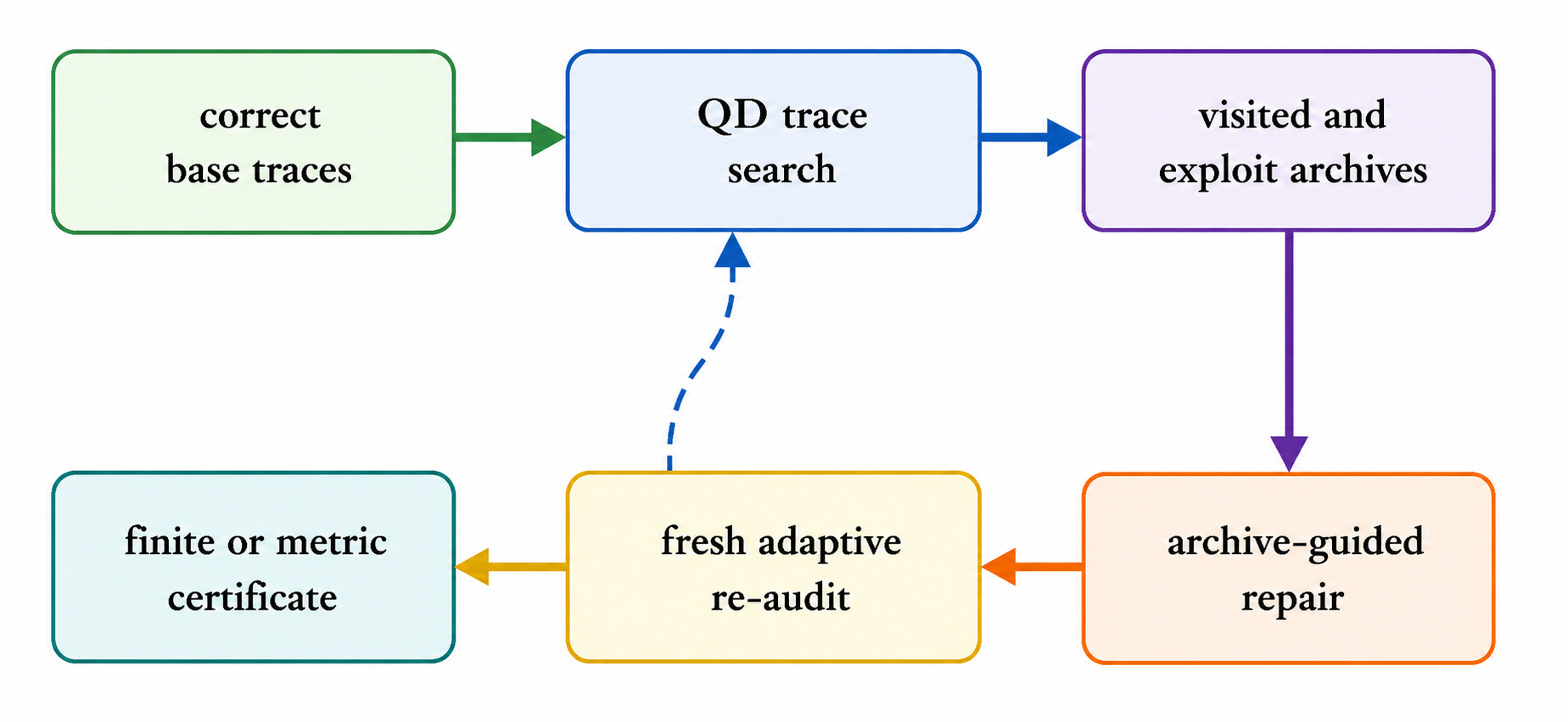}
    \caption{\textbf{The discovery-to-audit loop}. Search coverage and exploit coverage are recorded separately. Training on archived exploits is followed by a fresh adaptive search; operationally, a cell is treated as repaired only after that audit, while the theorem requires a valid cellwise upper bound. Finite-cell coverage controls an average and a tail, while a metric covering radius is needed for a worst-case certificate.}
    \label{fig:pipeline}
\end{figure}

The empirical evaluation consists of two parts. A controlled descriptor field makes every cell
severity known, allowing us to verify the finite accounting identities and to demonstrate directly
why a linear coverage-only worst-case claim fails. We then apply QD search to real PRMs on GSM8K
traces \citep{cobbe2021gsm8k}. The primary reproducible finding is a verification-dilution attack
against a mean readout of \texttt{Qwen2.5-Math-PRM-7B}. The same attack is structurally suppressed
by a minimum readout, although the larger run contains one strict min-readout exploit and therefore
does not support a claim of immunity. An RLHFlow value-head PRM behaves differently: positive
changes occur under both aggregations but are roughly fifty times smaller at their maximum. These
findings demonstrate why aggregation must be evaluated jointly with the PRM rather than presented
as a universal defense. The overall discovery-to-audit workflow is illustrated in
Figure~\ref{fig:pipeline}.

The paper makes four main contributions:

\begin{itemize}
    \item \textbf{A PRM-specific QD formulation} that preserves diverse correctness-flipping attacks while explicitly distinguishing visitation coverage from exploit coverage.

    \item \textbf{A robustness certification framework} comprising finite-cell average and tail certificates, an impossibility result for fraction-only worst-case guarantees, and a metric-cover theorem specifying the additional geometric structure required for valid worst-case certification.

    \item \textbf{A controlled diagnostic} that exposes the distinction between finite-cell guarantees and metric-cover worst-case guarantees.

    \item \textbf{An empirical evaluation} identifying an aggregation-dependent vulnerability, validating it through a syntactic negative control, analyzing cross-model transfer, and separating parser failures and numerically tiny effects from substantive robustness conclusions.
\end{itemize}

\section{Related Work}
\label{sec:related}

Process reward models (PRMs) were introduced to provide denser supervision than outcome-only reward
models and are now widely used for step-level verification, best-of-$n$ selection, tree search, and
process-based reinforcement learning
\citep{uesato2022solving,lightman2024lets,wang2024mathshepherd}. Their quality is typically
evaluated through error localization or by measuring how well they rank correct reasoning traces
above incorrect ones \citep{lambert2025rewardbench}. While these evaluations are essential, they do
not directly examine how a fixed PRM behaves when trace construction is optimized against its own
score. Our experiments instead condition on a verified correctness change and measure whether the
attacked trace receives a higher aggregate PRM score.

Reward hacking has been studied extensively in the context of outcome reward models. Optimizing a
policy against a fixed learned reward can ultimately reduce true task performance
\citep{gao2023scaling}. Adversarial training provides a general framework for improving robustness
against adversarial attacks \citep{madry2018towards}, and recent work has adapted these ideas to
reward models specifically \citep{bukharin2025adversarialrm}. More recent work jointly trains a
generator and a PRM to produce increasingly challenging process-level negatives
\citep{juneja2025adversarialprm}. Our focus is complementary. Rather than treating adversarial
training alone as evidence of robustness, we study how a diverse, explicitly indexed failure
archive changes what can be audited and what can be certified.

Additional discussion of related aggregation choices, quality-diversity search, and the distinction between our audited failure setting and prior PRM stress-testing approaches is provided in Appendix~\ref{app:ex_related_work}.

\section{Problem Formulation}
\label{sec:setup}

Let $x=(u,s)$ contain a problem $u$ and a reasoning trace $s$, and let
$q(x)\in\{0,1\}$ denote final-answer correctness under a deterministic task verifier. A PRM with
parameters $\theta$ emits step scores, which an aggregation rule converts into a trace score
$S_\theta(x)\in[0,1]$. The aggregation rule is part of the audited system: changing a mean readout
to a minimum readout changes $S_\theta$ even when the underlying step scores are unchanged.

An attack operator $\tau\in\X$ maps $x$ to an edited trace $\tau(x)$. We restrict the primary audit
to correct base traces, so $q(x)=1$, and verify the attacked answer independently. This makes the
correctness event discrete and prevents a score difference from being conflated with the unit-sized
change in the binary label.

\begin{definition}[Correctness-flipping exploit gain]
\label{def:gain}
The exploit gain of an attack instance $z=(x,\tau)$ is
\begin{align*}
\ell_\theta(z)
={}&\ind\{q(x)=1,\ q(\tau(x))=0\}\\
&\times\big[S_\theta(\tau(x))-S_\theta(x)\big]_+,
\end{align*}
where $[a]_+=\max(a,0)$. For an operational threshold $\xi\ge0$, $z$ is a
$\xi$-exploit when the correctness indicator is one and
$S_\theta(\tau(x))-S_\theta(x)>\xi$.
\end{definition}

This definition separates two facts that should not be added together: the attack has made the
answer wrong, and the proxy has moved in the wrong direction. Because $S_\theta\in[0,1]$, every
cell severity defined below also lies in $[0,1]$, making the reported maximum increases of $0.294$
and $0.005$ directly interpretable. The experiments report the strict convention $\xi=0$ together
with effect magnitudes; in deployment, $\xi$ should be set above reader and numerical tolerance.

Each attack has a behavioral descriptor $\beta(z)\in\D$. We partition $\D$ into $M$ cells and
write $\X_c=\{z:\beta(z)=c\}$ for the attack instances represented by cell $c$. The population
severity of a cell is
\[
g_\theta(c)=\supop_{z\in\X_c}\ell_\theta(z),
\qquad
g_{\max}=\max_{c\in\D}g_\theta(c),
\]
with the supremum defined as zero when a cell contains no valid correctness-flipping attack. An
empirical search only supplies a lower bound on $g_\theta(c)$; finding one elite in a cell does not
establish that the cell supremum has been found.

The distinction leads to three coverage quantities. If $A$ is the evaluated archive, then
$V(A)$ contains cells in which at least one valid candidate was scored, while $C_\xi(A)$ contains
cells whose empirical elite is a $\xi$-exploit. Their respective fractions are
\[
\rho_{\mathrm{visit}}(A)=\frac{|V(A)|}{M},
\qquad
\rho_{\mathrm{exp}}(A)=\frac{|C_\xi(A)|}{M}.
\]
After repair, a third set $\R(A)$ contains cells assigned a valid upper bound at a declared residual
level $\varepsilon$. Its fraction $\rho_{\mathrm{rep}}=|\R(A)|/M$ is the coverage that enters the
finite-cell certificate. A finite adaptive re-audit supplies empirical evidence for this condition
but, without additional structure, does not prove a population supremum. A cell is not declared
repaired merely because one archived example from it was used in training.

\section{Quality-Diversity Discovery}
\label{sec:qd}

MAP-Elites maintains one elite in each descriptor cell. For a fixed base trace, the search is
initialized from operator-specific seeds. At each iteration, it samples an existing elite, mutates
its operator or magnitude descriptor, instantiates the corresponding textual edit, verifies the
resulting final answer, and queries the PRM. A candidate replaces the incumbent elite whenever it
achieves a larger exploit gain. Unlike severity-only optimization, the archive preserves the
strongest attack within every explored region of the descriptor space, even when that attack is not
globally optimal. The resulting archive exposes a diverse repertoire of failure modes while
providing training pairs for archive-guided repair.

Our real-model instantiation employs five operator families crossed with five magnitude levels. The
operators append a confident but unjustified conclusion, insert locally plausible verification
padding, wrap an incorrect conclusion in authoritative mathematical language, overwrite the final
numeric result with a nearby incorrect value, or repeat the incorrect answer as though repetition
constituted evidence. The magnitude descriptor controls the amount of inserted or modified content.
Appendix~\ref{app:qd} details the archive update rule, the descriptor definitions, and the
accounting procedure used when attacks derived from different base traces map to the same cell.

The descriptor space is intentionally interpretable, but its limitations must be acknowledged. A
$5\times5$ grid does not constitute a cover of natural-language trace space. Two edits assigned to
the same cell may differ substantially in their semantics, while entirely new operator families may
lie outside the descriptor grid. Consequently, $\rho_{\mathrm{exp}}$ characterizes exploit coverage
only within the audited operator--magnitude space rather than across all possible PRM failures.
Section~\ref{sec:certificate} shows that meaningful worst-case guarantees require additional
geometric structure, expressed through a metric over complete attack instances together with an
associated covering radius.

\section{Archive-Guided Repair and Coverage Certificates}
\label{sec:certificate}

The positive elites in an archive form adversarial training pairs. A direct repair objective adds a
pairwise margin loss to the original clean-data objective,
\begin{align*}
&\mathcal{L}_{\mathrm{repair}}(\theta)
=\mathcal{L}_{\mathrm{clean}}(\theta)\\
&\ +\frac{\lambda}{|A_+|}
\sum_{(x,\tau)\in A_+}
\bigl[S_\theta(\tau(x))-S_\theta(x)+m\bigr]_+,
\end{align*}
where $A_+$ contains verified exploits and $m\ge0$ is a desired separation margin. The clean term
anchors the original PRM behavior and prevents a vacuous repair that lowers all scores. After
training, the system is attacked again from fresh seeds. The following statements apply to the
post-repair model when a certification procedure supplies a declared cellwise residual; a heuristic
re-audit is its empirical approximation rather than a formal upper bound.

Let $g_{\theta'}(c)$ denote the post-repair severity. For a repaired set $\R\subseteq\D$, define
\begin{align*}
\varepsilon_{\R}&=\max_{c\in\R}g_{\theta'}(c),\qquad
U_{\R}=\max_{c\notin\R}g_\theta(c),\\
\zeta_{\R}&=\max_{c\notin\R}
\big[g_{\theta'}(c)-g_\theta(c)\big]_+,
\end{align*}
where a maximum over an empty set is zero. The spillover term $\zeta_{\R}$ records the largest
increase induced by training in an unrepaired cell; assuming it is zero would generally be
unjustified for a neural model.

\begin{theorem}[Finite-cell coverage certificate]
\label{thm:finite}
Let $\rho_{\mathrm{rep}}=|\R|/M$ and suppose a valid cellwise certificate establishes
$\varepsilon_{\R}\le\varepsilon$. Then
\[
\max_{c\in\D}g_{\theta'}(c)
\le
\max\!\left\{\varepsilon,
U_{\R}+\zeta_{\R}\right\}.
\]
Moreover,
\[
\frac{1}{M}\sum_{c\in\D}g_{\theta'}(c)
\le
\rho_{\mathrm{rep}}\varepsilon
+(1-\rho_{\mathrm{rep}})(U_{\R}+\zeta_{\R}),
\]
and the fraction of cells whose post-repair severity exceeds $\varepsilon$ is at most
$1-\rho_{\mathrm{rep}}$.
\end{theorem}

The theorem separates the claims cleanly. The first inequality is an exact residual decomposition:
the worst case is either the audited residual on repaired cells or the severity, including harmful
spillover, outside them. The second and third statements are the coverage-indexed conclusions that
follow from a cell fraction. They control an average and a tail, not the largest uncovered cell.

\begin{proposition}[A coverage fraction cannot control the worst cell]
\label{prop:impossible}
For every $M\ge2$ and every attainable $\rho<1$, there is an exchangeable cell-severity field and a
perfectly repaired set covering fraction $\rho$ for which the uncovered supremum equals
$g_{\max}$. Hence, even with zero repaired-cell error, no universal bound of the form
$\max_c g_{\theta'}(c)\le h(\rho)g_{\max}$ with $h(\rho)<1$ can follow from coverage fraction alone.
\end{proposition}

Proposition~\ref{prop:impossible} rules out the appealing but incorrect $(1-\rho)g_{\max}$
worst-case law. To obtain a worst-case guarantee that improves continuously with coverage, the
archive must say how far every possible attack lies from an audited example.

\begin{theorem}[Metric-cover certificate]
\label{thm:metric}
Let $\X_{\mathrm{flip}}$ be the set of valid correctness-flipping attack instances, let
$(\X_{\mathrm{flip}},d)$ be compact, and let $B\subset\X_{\mathrm{flip}}$ be a
finite audited set with covering radius
\[
r(B)=\sup_{z\in\X_{\mathrm{flip}}}\min_{b\in B}d(z,b).
\]
Suppose the post-repair exploit loss $\ell_{\theta'}$ is $L$-Lipschitz with respect to $d$ and the
audit establishes $\ell_{\theta'}(b)\le\varepsilon$ for every $b\in B$. Then
\[
\sup_{z\in\X_{\mathrm{flip}}}\ell_{\theta'}(z)
\le \varepsilon+Lr(B).
\]
\end{theorem}

This theorem states the additional burden hidden by a cell count. The metric must include the base
problem and the semantic attack, or the descriptor must be sufficient for the loss; a radius
computed only from operator identifiers does not control variation inside a cell. In practice,
$L$, $r(B)$, and the audit residual must be reported or upper-bounded. The complete proofs of
Theorem~\ref{thm:finite}, Proposition~\ref{prop:impossible}, and
Theorem~\ref{thm:metric} appear in Appendix~\ref{app:proof}.

The aggregation mechanism found in the real-model search has a separate algebraic explanation.

\begin{proposition}[Padding under mean and minimum readouts]
\label{prop:pooling}
Let a trace have $k$ step rewards with mean $\bar r$ and minimum $r_{\min}$. Append $m$ steps whose
mean reward is $\bar u$ and minimum reward is $u_{\min}$. The mean score changes by
\[
\frac{m}{k+m}(\bar u-\bar r),
\]
whereas the minimum score becomes $\min(r_{\min},u_{\min})$ and therefore cannot increase under an
append-only edit.
\end{proposition}

The proposition does not imply that minimum aggregation is universally robust. An edit that removes
or rewrites the original minimum-scoring step can still raise the minimum, and a model can assign a
high score to every incorrect step. It predicts only that high-scoring padding can dilute a mean but
cannot hide an existing low score from a minimum.

\section{Experiments}
\label{sec:exp}
The experiments are organized around the distinctions introduced above. We first examine what a
cell fraction reveals when the population severity is fully known. We then evaluate whether the QD
procedure discovers correctness-flipping score increases on a real PRM, whether the mechanism
persists under a different aggregation rule, and whether the resulting attacked traces transfer
across independently trained reward heads. The controlled field serves as an accounting experiment
rather than a surrogate for neural fine-tuning, whereas the real-model study comprises both a
discovery audit \emph{and} an executed repair evaluated by fresh adaptive re-attack, including a
clean-split rerun that tests generalization to unseen problems
(\S\ref{sec:realrepair}). Maintaining this distinction prevents a simulated repair from being
presented as evidence of an empirical defense.

\paragraph{Controlled coverage on a known field}
\label{sec:sim}

We begin with a $24\times24$ descriptor grid comprising $576$ cells. Each cell is assigned a
nonnegative latent severity generated from a smooth mixture of Gaussian bumps and low-amplitude
noise, producing a small number of high-severity regions alongside many benign ones. The original
field is normalized by its raw maximum of $1.201$, yielding severities in $[0,1]$ consistent with
Definition~\ref{def:gain}. MAP-Elites explores the grid through local descriptor mutations. At
each evaluation budget, an oracle repair sets the severity of archived cells to the corresponding
normalized floor $\varepsilon=0.0167$ while leaving all remaining cells unchanged. Because this
intervention is defined at the cell level, it directly tests the finite-cell accounting of
Theorem~\ref{thm:finite}; it does not evaluate whether gradient-based training would satisfy the
same condition.

\begin{table}[t]
\centering
\small
\resizebox{\columnwidth}{!}{%
\begin{tabular}{rrrrc}
\toprule
\textbf{QD evals} & \textbf{$\rho$} & \textbf{post-repair sup} & \textbf{uncovered sup} & \textbf{certificate} \\
\midrule
200  & 0.29  & 1.0000 & 1.0000 & holds \\
500  & 0.55  & 0.9542 & 0.9542 & holds \\
1000 & 0.82  & 0.9542 & 0.9542 & holds \\
2000 & 0.96  & 0.6162 & 0.6162 & holds \\
4000 & 0.998 & 0.1490 & 0.1490 & holds \\
8000 & 1.00  & 0.0167 & 0.0000 & holds \\
\bottomrule
\end{tabular}}
\caption{Controlled field after normalizing the raw severities by their original maximum $1.201$,
so that the pre-repair $g_{\max}=1$. The post-repair supremum equals the maximum of the
optimization floor and the uncovered supremum, exactly as stated by the finite-cell certificate.
High cell coverage does not imply a proportionally small worst remaining cell: even at
$\rho=0.96$, the residual remains $0.6162$.}
\label{tab:sim}
\end{table}

Table~\ref{tab:sim} illustrates two complementary observations. First, the finite-cell residual
certificate holds exactly at every evaluation checkpoint, and the post-repair supremum decreases
monotonically as increasingly severe cells are eventually covered. Second, this decrease is not
linear in coverage. Between coverages $0.55$ and $0.82$, the worst remaining severity is
unchanged, and even at $\rho=0.998$, the single uncovered region retains severity $0.1490$. The
regression of the post-repair supremum on the uncovered supremum has slope $0.99$ and
$R^2=0.9999$ after including the optimization-floor row, serving as an implementation check of the
accounting rather than evidence for a proportional worst-case law. Complete construction and
reporting details are provided in Appendix~\ref{app:sim}.

\paragraph{Real-model protocol}
\label{sec:realsetup}

The primary audit evaluates \texttt{Qwen2.5-Math-PRM-7B}
\citep{zhang2025qwenprmlessons}. We apply the model's chat template and
step-separator token, reading the positive-class softmax probability at each separator rather than
using raw logits. Base examples are correct multi-step GSM8K solutions. An attack is counted only
when exact-match verification confirms that the base answer is correct, the edited answer is
incorrect, and the aggregate PRM score increases strictly. Each trace is initialized with five
operator-specific seeds followed by $40$ mutation evaluations, yielding $45$ evaluated candidates
per readout. Counts are therefore candidate-level descriptive statistics, whereas archive coverage
counts distinct descriptor cells and is bounded above by $25$.

We evaluate both the mean and minimum of the emitted step probabilities. A matched syntactic
control (arithmetic corruption, dropped steps, operand swaps, and contentless padding) tests
whether the search rewards arbitrary edits. The current evaluation adopts the strict threshold
$\xi=0$. Because numerically tiny floating-point changes satisfy this convention, we report the
maximum gain and interpret near-threshold counts cautiously. Appendix
\ref{app:realprm} details the reader, correctness verification, archive construction, and
thresholding rules.

\paragraph{An aggregation-dependent Qwen vulnerability}
\label{sec:realprm}

The initial run attacks the $67$ of $80$ sampled traces whose base
answers pass exact-match verification: $16$ strict exploits under mean
pooling, none under minimum, coverage $0.16$, and mean rise $0.039$
(Table~\ref{tab:realprm}, appendix). Elites are dominated by
verification dilution: a plausible incorrect final answer is followed
by locally fluent verification steps whose scores exceed the trace
mean. Proposition~\ref{prop:pooling} predicts this behavior exactly.
The same padding cannot increase an existing minimum, explaining the
aggregation contrast without interpreting a zero count as a robustness
guarantee.

\paragraph{Syntactic control: padding, not deception, is the mechanism.}
A matched control replaces the deceptive operators with five syntactic
operators under the identical protocol and budget: $47$ strict
mean-readout exploits, \emph{all} from the two padding operators and
none from the three pure-corruption operators. Every control edit
applies the same fixed final-number corruption (ensuring a verified
correctness flip), while the score \emph{increase} arises from padding
alone. Thus, any locally innocuous padding exploits mean pooling
(Proposition~\ref{prop:pooling}), superseding the earlier
deception-required interpretation (appendix Syntactic control).

The larger run over $120$ traces reproduces both the mechanism and its
boundary: $44$ strict mean-pooling exploits, ten occupied exploit cells
($\rho_{\mathrm{exp}}=0.40$), mean rise $0.087$, and maximum $0.294$.
Minimum pooling yields one strict exploit rather than none, indicating
that the minimum readout strongly suppresses the tested padding
mechanism on this PRM without providing immunity.

\paragraph{Problem-level uncertainty and audit thresholds.}
Using the problem as the resampling unit, the mean-readout exploit rate
is $0.168$ $[0.103,0.243]$ ($n{=}107$), with a best-gap $95$th
percentile of $0.097$ $[0.029,0.237]$. The threshold curve
($\tau=0.01/0.02/0.05/0.1/0.2$: $0.150/0.140/0.084/0.037/0.028$)
shows that strict counts overstate the number of exploits surviving
practical margins.
\paragraph{Equal-budget baselines.}
At the identical $45$-evaluation budget, random perturbation and
exhaustive grid enumeration occupy \emph{more} exploit cells than
MAP-Elites ($0.48$ vs.\ $0.40$; single-objective collapses to
$0.12$). With only $25$ cells, exhaustive evaluation is inexpensive,
so this experiment does \emph{not} demonstrate that QD search is
necessary to achieve coverage at this grid size. The reported severity
statistics follow different conventions across methods (baselines:
nonnegative per-problem best gain, mean $0.017$--$0.018$; archive:
signed per-problem best gaps, exploited-candidate mean rise $0.087$,
maximum $0.294$), and we therefore make no severity claim based on
these columns. The corresponding same-statistic comparison is
\emph{performed} in \S\ref{sec:pairedrepair}
(Table~\ref{tab:cleanutil}), where equal-budget exhaustive enumeration
slightly exceeds MAP-Elites. The contribution of the archive is instead
its behavior-indexed structure, which supports the repair and
certificate diagnostics.

\paragraph{Model dependence and transfer}
\label{sec:multiprm}

The same reader-normalized search on an RLHFlow Llama-3.1 value-head
PRM \citep{xiong2024rlhflow} provides an informative contrast.
Hundreds of strictly positive candidates are observed under both
readouts, and every descriptor cell is occupied, yet the maximum score
increase is only $0.005$---approximately fifty times smaller than for
Qwen. Consequently, a deployment threshold above this scale may
eliminate most or all strict events. The results therefore indicate
model dependence without providing evidence of severe RLHFlow
exploitability (full accounting in Table~\ref{tab:multiprm},
appendix).

Transfer re-scores the strict mean-readout exploits of one PRM on the
other. A transfer is counted when the target PRM also strictly prefers
the incorrect edited trace over its correct base. All Qwen exploits
transfer to RLHFlow, and $0.93$ of RLHFlow exploits transfer to Qwen
(Table~\ref{tab:transfer}, appendix). The attack templates are
therefore not specific to a single reader, although the strict
threshold and the small RLHFlow effects preclude a stronger
operational-severity claim.

An attempted Skywork-o1 audit \citep{skywork2024prm} did not produce a
valid model result because the uniform reader failed to parse the
model's native step-reward channel, returning constant zeros
throughout. These zeros, together with the corresponding zero transfer
entries, are excluded from the evidential tables. Appendix
\ref{app:readerfailure} documents the failed run to avoid its
misinterpretation as evidence of robustness. A meaningful Skywork
comparison requires a model-specific reader validated against reference
outputs before attack evaluation.

\paragraph{Certificate diagnostics on the real archive (plug-in, not certified)}
\label{sec:realcert}

We instantiate both certificate formulas on Qwen's measured mean-readout
archive as \emph{plug-in diagnostics}, explicitly \emph{not} certified
worst-case residuals, because five inputs are empirical lower bounds or
unverified assumptions (searched severities, the observed ratio
$\widehat L=0.46$, a descriptor-cell metric,
$\rho_{\mathrm{exp}}$ in place of $\rho_{\mathrm{rep}}$, and unbounded
spillover; the complete list appears in the
caption of Table~\ref{tab:realcert}). The plug-in arithmetic nevertheless
illustrates the central point: ten exploit cells
($\rho_{\mathrm{exp}}=0.40$, $g_{\max}=0.294$), zero measured severity
elsewhere, and a finite-cell residual
$\varepsilon_{\mathrm{opt}}=0.02$; the metric-cover plug-in remains
$0.48$ for \emph{every} audited subset smaller than the full grid
(radius $1$ at $|B|=5$--$20$), decreasing to $0.02$ only when all $25$
cells are audited (radius $0$). Audited-cell \emph{count} alone
provides no improvement until the audited set forms a genuine cover.
Values are taken from \texttt{results\_certificate/}.

\paragraph{Executed archive-guided repair on the real PRM}
\label{sec:realrepair}

We now execute the repair stage on \texttt{Qwen2.5-Math-PRM-7B} rather
than modeling it. The first \emph{implemented} objective is simpler
than the pairwise-margin loss of
\S\ref{sec:certificate}: a LoRA adapter (rank $8$,
$\alpha{=}16$, dropout $0.05$, query/value) is trained for $60$ AdamW
steps (lr $10^{-4}$, batch $1$) to down-weight archived exploits
through the model's own readout, without a margin or clean term;
clean behavior is evaluated post hoc. Against a fresh adaptive
re-attack, the comparison is intentionally \emph{unpaired}: full-pool
discovery identifies $44$ exploits ($0.168$ of $107$), whereas the
post-repair re-attack finds $1$ (maximum gain $0.0012$ versus the
pre-repair $0.294$), while clean ranking remains $58/58$; the matched
paired baseline is reported in \S\ref{sec:pairedrepair}. The stricter
clean-split rerun (repair on $13$ repair-half exploits followed by
re-attack on the disjoint half) likewise yields $1$ strict exploit,
maximum gain $0.0092$, and clean ranking $58/58$. The pipeline cannot
repair the RLHFlow value-head PRM because it is not differentiable in
our reader; this limitation is reported rather than omitted. Values are
taken from \texttt{results\_defense/} and
\texttt{results\_gsm8k\_cleansplit/}.

\paragraph{Benchmark generalization: MATH-500.}
Applying the identical protocol to MATH-500
\citep[the][subset]{hendrycks2021math,lightman2024lets} ($82$ of $120$
sampled traces attacked) reproduces the qualitative findings on harder
problems: $41$ strict mean-readout exploits (problem-level rate
$0.232$ $[0.146,0.329]$), exploit-cell coverage $0.32$, and a worst-cell
gap of $0.232$; the threshold curve decreases from $0.207$ at
$\tau{=}0.01$ to $0.024$ at $\tau{=}0.1$. Minimum pooling yields
$19$ strict exploits whose largest cell gap, $0.020$, is an order of
magnitude smaller than the mean-readout value of $0.232$, indicating
that on harder problems minimum pooling suppresses \emph{severity}
rather than exploit counts. The post-repair fresh re-attack finds
$0$ exploits (coverage $0$; $g_{\max}=0$ under the nonnegative-gain
convention, maximum \emph{signed} difference $-0.0003$), while all
$40/40$ held-out clean pairs remain correctly ranked. Values are taken
from \texttt{results\_math500/}.

\paragraph{Paired repair with the method's loss}
\label{sec:pairedrepair}

The predeclared paired protocol (endpoints fixed before results; no
public preregistration) addresses the remaining evaluation gaps on
GSM8K: a $53/54$ repair/evaluation split, discovery restricted to the
repair half ($20$ exploits, all
\texttt{verify\_dilution}), and matched pre/post fresh re-attacks using
the method's margin-plus-clean-anchor loss
($0.10$/$\lambda{=}0.5$, cached clean anchors; LoRA
r$8$/$\alpha16$, AdamW $10^{-4}$, $60$ steps; $3$ repair
$\times$ $2$ attack seeds). The frozen primary endpoint is the
per-problem best nonnegative gain averaged over attack seeds, ensuring
that every reported column is internally consistent
(Table~\ref{tab:paired}): the pre-repair rate is $8/54=0.148$, the
post-repair rate is $2/54$--$4/54$, and all paired confidence intervals
exclude zero ($\Delta$gain $-0.012$ to $-0.015$; hierarchical
$-0.014$ $[-0.028,-0.003]$). Residual severity also decreases: the
maximum gain falls from $0.333$ to $0.177$--$0.212$
(q95 $0.146\to\le0.034$), indicating reductions in both prevalence and
the upper tail, although worst-case attacks with approximately
two-thirds of the original severity remain. AUROC improves
($0.884\to0.915$--$0.935$), best-of-$4$ remains $1.00$, and the Brier
score worsens ($0.342\to0.385$ worst). 
The v2 run additionally evaluates \emph{repair-source} controls under
the same pair budget and optimization steps. Random, exhaustive, and
strongest-only attack sets all achieve $2/54$ post-repair exploits,
with reductions at least as large as those obtained from the QD archive
($-0.016$ to $-0.017$), whereas the clean-only control produces no
change ($\Delta\equiv0$). These results indicate that adversarial
fine-tuning drives the repair, whereas our claim for the archive is
limited to its audit structure. The unseen-wording evaluation
(variant-B templates, never used for training) is only weakly
informative: variant B produces few exploits even before repair (rate
$0.037$), and post-repair rates are $0$--$0.019$. This shows no
regression but provides insufficient pre-repair signal to support a
wording-generalization claim. Math-Shepherd likewise has little to
repair (rate $0.019$). Complete results are reported in
Tables~\ref{tab:cleanutil}--\ref{tab:paired} (appendix;
\texttt{results\_paired\_v2/}). Additional implementation details, repair protocols, control experiments, and supplementary analyses are provided in Appendix~\ref{app:repairprotocol}. Supplementary figures illustrating repair behavior, aggregation effects, threshold sensitivity, generalization, and clean-utility evaluation are presented in Appendix~\ref{sec:supp_figures}.

\section{Discussion and Conclusion}
\label{sec:discussion}

The observed pattern is more precise than the broad claim that ``PRMs
are easily hacked'': padding exploits a mean readout when inserted
scores exceed the trace average, whereas a minimum readout blocks that
route but not lowest-step edits or a confidently wrong PRM. These
results indicate that PRM vulnerabilities are not governed by a single
failure mode, but emerge from the interaction between attack structure,
score aggregation, and the reliability of the learned evaluator.
Coverage describes a repertoire rather than a worst-case guarantee
(Figure~\ref{fig:coverage}); certificate numbers are plug-in
diagnostics (\S\ref{sec:realcert}) rather than standalone robustness
certificates. Repairs remain effective on unseen problems, MATH-500,
and the paired protocol, with the repair-source controls attributing the
effect to adversarial fine-tuning rather than the archive
(\S\ref{sec:pairedrepair}). Together, these findings support adaptive
stress testing and repair as complementary components for evaluating
the reliability of process reward models under adversarial
optimization.

\section*{Limitations}

The real study uses modest problem subsets ($107$--$120$ traces per benchmark on GSM8K and
MATH-500) and a hand-designed $5\times5$ descriptor grid. The strict threshold $\xi=0$ is
sensitive to small reader differences, particularly for RLHFlow, and candidate counts are not
independent because multiple mutations share a base problem and archive lineage
(problem-level bootstrap intervals and threshold curves are reported in
\S\ref{sec:realprm}). The real-archive certificate numbers are plug-in
diagnostics, not certified bounds (\S\ref{sec:realcert}). The paired v2
protocol (\S\ref{sec:pairedrepair}) supplies matched pre/post
re-attacks, reconciled union statistics, hierarchical bootstrap over
problems, repair seeds, and attack seeds, equal-budget repair-source
controls, and a full clean-utility panel; its remaining gaps are honest
and specific: worst observed post-repair attacks retain about
two-thirds of the original maximum gain; the Brier score degrades
slightly post-repair (score inflation on clean and corrupted traces
alike); the repair-source controls show the QD archive is not a better
repair-data source than random or exhaustive attack sets at this scale,
so the quality-diversity claim concerns the auditable archive, not
search or repair superiority; and the unseen-wording arm is weakly
informative because the variant templates barely exploit even the
original model, while a genuinely open-ended attack generator remains
unexecuted.  All exploits in this channel come from one operator family
(\texttt{verify\_dilution}), so cross-family repair generalization is
unmeasured.  The clean-split repair
establishes generalization to unseen \emph{problems}, but not to unseen
templates or operator families (Appendix~\ref{app:repairprotocol}); the
repair pipeline does not apply to non-differentiable readouts such as
the RLHFlow value head.
ProcessBench~\citep{zheng2024processbench}/PRMBench~\citep{song2025prmbench}
discrimination and a downstream reasoning evaluation of the repaired PRM
are specified but not run; no number for either appears in this paper. Finally, the edit operators are
textual and black-box; gradient-, logit-, or policy-guided adversaries could expose failures
outside the descriptor family. These limitations bound the empirical claim to a reproducible
diagnostic finding and motivate the broader evaluation protocol in the appendix.

\bibliography{references}

\appendix

\section{Extended Related Work}
\label{app:ex_related_work}

Aggregation is a well-established determinant of PRM behavior. Sum- or accumulation-based credit
assignment can reward long or repetitive generations, whereas minimum-based aggregation can
suppress some of these incentives \citep{cheng2025stopsummation}. We do not claim minimum
aggregation as a methodological contribution. Instead, the mean--minimum comparison serves as a
controlled diagnostic for the mechanism uncovered by the QD search, while the RLHFlow results
demonstrate that the effect is model dependent. This distinction is important because a minimum
readout may itself be sensitive to a single spuriously low score and may not preserve the ranking
utility of the original PRM.

Our search framework builds on quality-diversity (QD) optimization. MAP-Elites and related
illumination algorithms preserve high-performing solutions throughout a user-defined behavior space
\citep{mouret2015illuminating,pugh2016quality,lehman2011abandoning}, extending ideas originally
developed for learning diverse behavioral repertoires \citep{cully2015robots}. QD has also been
applied to language-model red teaming. Rainbow Teaming formulates adversarial prompt generation as
open-ended QD search, and subsequent work learns behavior-conditioned attackers across multiple
risk categories \citep{samvelyan2024rainbow,wang2025qdrt,samvelyan2024rainbowplus}. Our contribution
is therefore not the general observation that QD can diversify attacks, nor the claim that
MAP-Elites is inherently a superior search algorithm or repair-data source. Indeed, under our
small fixed operator grid and equal search budget, exhaustive enumeration matches MAP-Elites on
both discovery and repair performance (\S\ref{sec:pairedrepair}). Rather, our contribution is the
specialization of QD to process-reward traces, the formulation of the behavior-indexed
\emph{archive as an auditable object}, the explicit separation of distinct coverage notions, and
the characterization of which forms of coverage support which robustness guarantees.

The closest PRM-specific stress-testing framework is EST-PRM, which measures score inflation and
correlation collapse under fixed, label-preserving transformations such as step inflation,
dependency-aware reordering, and confidence markers \citep{shihab2026estprm}. Our setting changes
the audited event. We begin with correct reasoning traces and search for edits that simultaneously
flip final correctness and increase the aggregate PRM score. The resulting QD archive therefore
consists of verified proxy-improving failures rather than a label-preserving invariance suite. The
overlap in structural perturbations remains important, and we accordingly use those
transformations as informed operator seeds rather than presenting them as novel attack classes.

\section{Complete Proofs}
\label{app:proof}

We restate the relevant notation. The finite descriptor partition contains $M$ cells. The
pre-repair and post-repair population severities are $g_\theta(c)$ and $g_{\theta'}(c)$,
respectively. The repaired set is $\R$, its coverage is
$\rho_{\mathrm{rep}}=|\R|/M$, and
\[
\varepsilon_{\R}=\max_{c\in\R}g_{\theta'}(c),
\qquad
U_{\R}=\max_{c\notin\R}g_\theta(c),
\]
\[
\zeta_{\R}=\max_{c\notin\R}
\big[g_{\theta'}(c)-g_\theta(c)\big]_+.
\]
All severities are nonnegative. We use the convention that a maximum over the empty set is zero.

\paragraph{Proof of the finite-cell certificate}

\begin{proof}[Proof of Theorem~\ref{thm:finite}]
Partition the descriptor cells into $\R$ and its complement. For every $c\in\R$, the post-repair
audit gives
\[
g_{\theta'}(c)\le\varepsilon_{\R}\le\varepsilon.
\]
For every $c\notin\R$, write the post-repair severity as its pre-repair value plus its change:
\begin{align*}
g_{\theta'}(c)
&=g_\theta(c)+g_{\theta'}(c)-g_\theta(c)\\
&\le g_\theta(c)+\big[g_{\theta'}(c)-g_\theta(c)\big]_+\\
&\le U_{\R}+\zeta_{\R}.
\end{align*}
Taking the maximum over the two parts of the partition gives
\[
\max_{c\in\D}g_{\theta'}(c)
\le
\max\!\left\{\varepsilon,U_{\R}+\zeta_{\R}\right\}.
\]

For the average, sum the corresponding cellwise bounds:
\begin{align*}
\frac{1}{M}\sum_{c\in\D}g_{\theta'}(c)
&=
\frac{1}{M}\sum_{c\in\R}g_{\theta'}(c)
+\frac{1}{M}\sum_{c\notin\R}g_{\theta'}(c)\\
&\le
\frac{|\R|}{M}\varepsilon
+\frac{M-|\R|}{M}(U_{\R}+\zeta_{\R})\\
&=
\rho_{\mathrm{rep}}\varepsilon
+(1-\rho_{\mathrm{rep}})(U_{\R}+\zeta_{\R}).
\end{align*}

Finally, every cell in $\R$ has post-repair severity at most $\varepsilon$. Therefore
\[
\{c:g_{\theta'}(c)>\varepsilon\}
\subseteq \D\setminus\R,
\]
and hence
\[
\frac{1}{M}
\left|\{c:g_{\theta'}(c)>\varepsilon\}\right|
\le\frac{M-|\R|}{M}
=1-\rho_{\mathrm{rep}}.
\]
All three conclusions are deterministic and make no exchangeability assumption.
\end{proof}

If a deployment distribution is uniform over cells, the average inequality is also an expected-risk
bound. More generally, for declared cell weights $w_c\ge0$ with $\sum_cw_c=1$, the same proof gives
\[
\sum_c w_c g_{\theta'}(c)
\le
w(\R)\varepsilon
+\sum_{c\notin\R}w_c
\left(g_\theta(c)+\zeta_c\right),
\]
where $w(\R)=\sum_{c\in\R}w_c$ and
$\zeta_c=[g_{\theta'}(c)-g_\theta(c)]_+$. This weighted form is preferable when attack cells are not
equally likely in deployment.

\paragraph{Proof that fraction coverage cannot control a supremum}

\begin{proof}[Proof of Proposition~\ref{prop:impossible}]
Fix $M\ge2$ and an attainable fraction $\rho=k/M<1$. Let every pre-repair cell have the same
severity,
\[
g_\theta(c)=g_{\max}>0
\qquad\text{for all }c\in\D.
\]
This constant field is exchangeable: permuting cell labels leaves it unchanged, and no cell is more
likely than another to contain a maximum because every cell is a maximum. Choose any repaired set
$\R$ of size $k$ and suppose, favorably, that repair drives every cell in $\R$ to zero and causes no
spillover. Since $k<M$, at least one cell remains outside $\R$, and every such cell retains severity
$g_{\max}$. Consequently,
\[
\max_{c\notin\R}g_{\theta'}(c)=g_{\max}.
\]
For any proposed universal multiplier $h(\rho)<1$, this construction violates
$\max_c g_{\theta'}(c)\le h(\rho)g_{\max}$ even though the field and the choice among equally severe
cells are fully symmetric. Adding an optimization floor $\varepsilon<{1-h(\rho)}g_{\max}$ does not
change the contradiction. Thus coverage fraction alone cannot yield a nontrivial worst-case
contraction.
\end{proof}

The argument also identifies the error in a quantile-based intuition. Selecting or repairing a
$\rho$ fraction of cells may control the distribution of a uniformly sampled cell, but the maximum
over the remaining $(1-\rho)M$ cells is not the $(1-\rho)$-quantile of the original severity
distribution. It is an extreme order statistic of the uncovered subset and can remain equal to the
global maximum.

\paragraph{Proof of the metric-cover certificate}

\begin{proof}[Proof of Theorem~\ref{thm:metric}]
Fix any $z\in\X_{\mathrm{flip}}$. Because $B$ is finite, there exists
$b_z\in B$ attaining $\min_{b\in B}d(z,b)$. Lipschitz continuity and the audited residual bound give
\begin{align*}
\ell_{\theta'}(z)
&\le \ell_{\theta'}(b_z)
+\left|\ell_{\theta'}(z)-\ell_{\theta'}(b_z)\right|\\
&\le \varepsilon+L d(z,b_z)\\
&\le \varepsilon+Lr(B).
\end{align*}
Since the inequality holds for every $z\in\X_{\mathrm{flip}}$, taking the supremum proves
\[
\sup_{z\in\X_{\mathrm{flip}}}\ell_{\theta'}(z)
\le\varepsilon+Lr(B).
\]
Compactness ensures that the covering radius is finite for the intended bounded domain; the
pointwise argument itself requires only that the displayed radius be finite.
\end{proof}

If the audit only establishes noisy upper estimates
$\widehat\ell_{\theta'}(b)$ with simultaneous error at most $\nu$, then
$\ell_{\theta'}(b)\le\widehat\ell_{\theta'}(b)+\nu$. Replacing $\varepsilon$ by
$\max_b\widehat\ell_{\theta'}(b)+\nu$ in the proof yields
\[
\sup_{z\in\X_{\mathrm{flip}}}\ell_{\theta'}(z)
\le
\max_{b\in B}\widehat\ell_{\theta'}(b)+\nu+Lr(B).
\]
This extension makes explicit that statistical uncertainty and geometric coverage are separate
terms.

\paragraph{Proof of the pooling proposition}

\begin{proof}[Proof of Proposition~\ref{prop:pooling}]
Let the original rewards be $r_1,\ldots,r_k$ and the appended rewards be
$u_1,\ldots,u_m$. By definition,
\[
\sum_{i=1}^k r_i=k\bar r,
\qquad
\sum_{j=1}^m u_j=m\bar u.
\]
The new mean is therefore
\[
\bar r'
=\frac{k\bar r+m\bar u}{k+m}.
\]
Subtracting the original mean gives
\[
\bar r'-\bar r
=\frac{k\bar r+m\bar u-(k+m)\bar r}{k+m}
=\frac{m}{k+m}(\bar u-\bar r).
\]
Thus padding raises the mean exactly when the appended average exceeds the original average.

For the minimum readout, the new set of rewards is the union of the original and appended rewards,
so
\begin{align*}
r'_{\min}
&=\min\bigl(\min_i r_i,\min_j u_j\bigr)\\
&=\min(r_{\min},u_{\min})
\le r_{\min}.
\end{align*}
Hence an append-only edit cannot strictly increase the minimum score.
\end{proof}

\section{Quality-Diversity Search Details}
\label{app:qd}

\paragraph{Archive update}

For each base trace, the implementation first constructs one seed per operator family. An elite
record contains the edited trace, its operator and magnitude descriptor, the base and attacked
aggregate scores, the verified base and attacked answers, and the resulting exploit gain. The
following update is applied for each mutation evaluation.

\begin{table}[h]
\centering
\small
\begin{tabular}{rp{0.82\linewidth}}
\toprule
\textbf{Step} & \textbf{Operation} \\
\midrule
1 & Sample a nonempty archive cell uniformly and retrieve its elite. \\
2 & Mutate the operator identifier or move the magnitude by one neighboring level, with boundary reflection. \\
3 & Instantiate the mutated edit on the same base trace and verify that the edited answer is parseable. \\
4 & Query the PRM once, aggregate its native step probabilities, and compute Definition~\ref{def:gain}. \\
5 & Insert the candidate if its cell is empty; otherwise replace the elite only when the exploit gain is larger. \\
\bottomrule
\end{tabular}
\caption{MAP-Elites update used by the discrete trace-edit search. The model query budget, rather
than the number of successful exploits, is the appropriate unit for comparisons with alternative
search methods.}
\label{tab:algorithm}
\end{table}

When several base traces map to the same global descriptor cell, the global archive retains the
largest observed exploit gain and stores the originating problem identifier. Candidate-level
exploit counts include every evaluated candidate satisfying the criterion and can therefore exceed
the number of cells. Cell coverage counts each descriptor cell once. This distinction explains why
RLHFlow can yield hundreds of strict candidates while the archive coverage remains bounded by one. The MAP-Elites archive update is summarized in Table~\ref{tab:algorithm}.

\paragraph{Descriptor semantics}

The descriptor combines one of five operator families with one of five ordered magnitudes. The
operator families and the semantic role of magnitude are summarized in
Table~\ref{tab:operators}. Exact surface strings are implementation details and should be released
alongside the evaluation script so that paraphrase-level changes are not hidden by a shared family
name.

\begin{table*}[t]
\centering
\small
\begin{tabular}{p{0.19\linewidth}p{0.46\linewidth}p{0.24\linewidth}}
\toprule
\textbf{Operator family} & \textbf{Transformation} & \textbf{Magnitude axis} \\
\midrule
Confident conclusion & Replaces the final answer with a plausible wrong value and appends an unjustified confident conclusion. & Amount of concluding scaffolding. \\
Verification padding & Adds locally fluent checking or restatement steps around the wrong conclusion without introducing valid new evidence. & Number of inserted verification units. \\
Authoritative wrapper & Surrounds the wrong result with formal mathematical language and certainty markers. & Density of authoritative framing. \\
Nearby overwrite & Silently changes the final numeric result to a nearby wrong value while preserving the preceding derivation. & Size of the normalized numeric displacement. \\
Repeated assertion & Restates the wrong answer multiple times as though repetition strengthened its validity. & Number of repeated assertions. \\
\bottomrule
\end{tabular}
\caption{Behavior descriptors in the real-model archive. These cells organize the declared attack
family; they do not by themselves define a metric cover of all semantic trace edits.}
\label{tab:operators}
\end{table*}

\paragraph{Coverage accounting}

A cell enters $V(A)$ after one valid candidate in that cell has been scored, regardless of whether
the candidate is an exploit. A cell enters $C_\xi(A)$ only when its elite exceeds the declared
threshold. Operationally, a cell is proposed for $\R(A)$ only if a fresh adaptive re-audit,
initialized independently of the training elite, fails to exceed the residual level $\varepsilon$
within the declared budget. The formal finite-cell theorem additionally requires that this empirical
condition be promoted to a valid upper bound by an appropriate certification argument; the
metric-cover theorem gives one such route. Reporting only $|C_\xi(A)|/M$ loses two important facts: an empty exploit cell may
have been extensively searched without success, and an occupied exploit cell may contain many
unseen attacks that survive training.

\section{Controlled-Landscape Details}
\label{app:sim}

\begin{figure}[t]
\centering
\begin{tikzpicture}
\begin{axis}[
  width=\linewidth,
  height=4.7cm,
  xlabel={archive coverage $\rho$},
  ylabel={post-repair supremum},
  xmin=0.25,
  xmax=1.02,
  ymin=0,
  ymax=1.08,
  legend style={at={(0.03,0.03)},anchor=south west,font=\scriptsize,draw=none},
  tick label style={font=\scriptsize},
  label style={font=\small},
  every axis plot/.append style={thick}
]
\addplot[mark=*,blue] coordinates {
  (0.29,1.0000)(0.55,0.9542)(0.82,0.9542)(0.96,0.6162)(0.998,0.1490)(1.00,0.0167)
};
\addlegendentry{observed supremum}
\addplot[dashed,red,domain=0.29:1.0] {(1-x)+0.0167};
\addlegendentry{fraction-only heuristic}
\end{axis}
\end{tikzpicture}
\caption{Coverage and worst-cell severity need not be proportional. The dashed curve is the
invalid heuristic $(1-\rho)g_{\max}+\varepsilon$, not a bound. Every non-full-coverage checkpoint
lies above it, including coverages of $0.96$ and $0.998$. The experiment therefore illustrates
Proposition~\ref{prop:impossible} as well as the valid residual decomposition.}
\label{fig:coverage}
\end{figure}
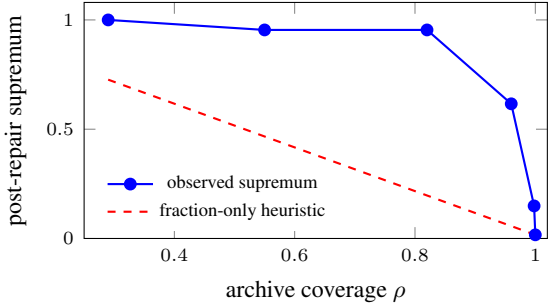

The synthetic field is a $24\times24$ grid. Its latent severities are formed from a small number of
smooth Gaussian bumps plus low-amplitude noise, followed by clipping at zero. MAP-Elites is
initialized along the operator axis and repeatedly selects an elite, changes one grid coordinate by
one cell, and evaluates the latent severity at the resulting location. Budgets are swept over
$\{200,500,1000,2000,4000,8000\}$.

The raw generated field has maximum $1.201$. Before reporting, every severity and the repair floor
are divided by $1.201$, producing the normalized values in Table~\ref{tab:sim}. At each checkpoint,
oracle repair replaces the severity of archived cells by $\varepsilon=0.0167$. This makes
$\zeta_{\R}=0$ and leaves the unrepaired cells at their pre-repair values. Consequently,
Theorem~\ref{thm:finite} specializes to
\[
\max_c g_{\theta'}(c)
=\max\!\left\{0.0167,\max_{c\notin\R}g_\theta(c)\right\}.
\]
The equality is expected because the experiment constructs the two terms directly. Its useful role
is to confirm the archive bookkeeping and to display the nonlinearity of the uncovered maximum.
All five non-full-coverage checkpoints in Table~\ref{tab:sim} violate the fraction-only heuristic
$(1-\rho)g_{\max}+0.0167$, while the finite-cell certificate holds at every checkpoint.

The released diagnostic files are \texttt{verify\_summary.json} and
\texttt{verify\_coverage.csv}. A complete release should additionally store the archive cell
identifiers at every budget, allowing $U_{\R}$ and any metric covering radius to be recomputed rather
than inferred from plotted values.

\section{Real-PRM Discovery Details}

\begin{table}[t]
\centering
\small
\resizebox{\columnwidth}{!}{%
\begin{tabular}{lrrrrr}
\toprule
\textbf{Setting} & \textbf{traces} & \textbf{evals} & \textbf{elites} & \textbf{exploits} & \textbf{$\rho_{\mathrm{exp}}$} \\
\midrule
Deceptive, mean & 67 & 3015 & 1213 & 16 & 0.16 \\
Deceptive, min  & 67 & 3015 & 1241 & 0  & 0.00 \\
\bottomrule
\end{tabular}}
\caption{Initial discovery audit on \texttt{Qwen2.5-Math-PRM-7B}. The evaluation budget is five
operator seeds plus $40$ mutation evaluations per attacked trace ($45\times67=3015$); ``elites''
counts the distinct occupied descriptor cells summed over traces, which is smaller because the
search revisits cells. Coverage is the fraction of the $25$
descriptor cells containing at least one strict exploit; it is not the fraction of natural-language
attack space covered.}
\label{tab:realprm}
\end{table}

\begin{table}[t]
\centering
\small
\begin{tabular}{lrrr}
\toprule
\textbf{Audited set $B$} & \textbf{$r(B)$} & \textbf{$\widehat L$} & \textbf{$\varepsilon_{\mathrm{opt}}+\widehat L\,r(B)$} \\
\midrule
Top-5 cells & 1.00 & 0.46 & 0.48 \\
Top-10 cells & 1.00 & 0.46 & 0.48 \\
Top-15 cells & 1.00 & 0.46 & 0.48 \\
Top-20 cells & 1.00 & 0.46 & 0.48 \\
All 25 (every op family) & 0.00 & 0.46 & 0.02 \\
\bottomrule
\end{tabular}
\caption{Metric-cover \emph{plug-in diagnostic} on the real Qwen archive
(not a certified bound; \S\ref{sec:realcert}). Five substitutions prevent
a certified reading: (i) search assigns non-exploit cells severity zero,
but a failed search is only a \emph{lower} bound on cell severity; (ii)
$\widehat L=0.46$ is the largest \emph{observed} loss ratio, a lower bound
on the true Lipschitz constant; (iii) the metric
$d=\ind\{\text{operator differs}\}+|\Delta\text{mag}|/4$ covers descriptor
cells, whereas the theorem requires coverage over complete problem--attack
instances, so $r(B)=0$ over the $25$ cells does not cover
natural-language attack space; (iv) $\rho_{\mathrm{exp}}$ is substituted
for $\rho_{\mathrm{rep}}$; (v) post-training spillover
$\zeta_{\mathcal R}$ is not bounded. The plug-in residual
$\varepsilon_{\mathrm{opt}}+\widehat L\,r(B)$ is additionally conditional
on the repair achieving $\varepsilon_{\mathrm{opt}}=0.02$ on $B$. Within
those assumptions it shrinks only when the covering radius does, i.e.\
when every operator family is audited, not merely when many cells are;
the finite-cell plug-in equals $\varepsilon_{\mathrm{opt}}$ because no
\emph{searched} uncovered cell carries an exploit.}
\label{tab:realcert}
\end{table}
\label{app:realprm}

\paragraph{Base traces and correctness}

The current runs use the first correct GSM8K training solutions with at least two identifiable
reasoning steps. Before an attack is scored, the base final answer is parsed and matched against the
dataset answer. The edited final answer is parsed through the same path and must be wrong. An
unparseable edit is a failed candidate, not an exploit. Because these examples come from a training
split and may have appeared in model pretraining, the experiment measures attack discovery on a
controlled trace source rather than out-of-distribution generalization.

Final-answer verification is sufficient for the correctness-flip event in Definition~\ref{def:gain}
but does not establish that every inserted sentence is locally coherent. A stronger release should
include an independent audit of fluency, semantic plausibility, and whether the first newly incorrect
step is located where the operator intended. Such validation can be performed on a stratified sample
without changing the deterministic final-answer label.

\paragraph{Qwen reader}

For \texttt{Qwen2.5-Math-PRM-7B}, the trace is formatted with the model's native chat template and
step separator. At each separator position, the implementation extracts the two-class reward-head
logits, applies a softmax, and retains the positive-class probability. Mean and minimum readouts are
computed over exactly the same extracted step sequence. The base and attacked traces must contain at
least one valid separator score. The reader should be validated against a small set of reference
examples from the model implementation before attack statistics are produced.

\paragraph{RLHFlow reader and normalization}

The RLHFlow Llama-3.1 PRM \citep{xiong2024rlhflow} loads as a causal LM
(\texttt{causal\_step} path): for each step $i$ the reader scores the
prefix through step $i$ followed by the literal query ``Is this step
correct? '' and takes the next-token probability
$P(\texttt{+})/(P(\texttt{+})+P(\texttt{-}))$, so every RLHFlow step
score is already a probability in $[0,1]$---the same scale as the Qwen
two-class softmax readout; no further normalization is applied to either
model, and mean/min readouts aggregate these per-step probabilities
directly.

\paragraph{Thresholds and dependence}

The existing tables use $\xi=0$, so any strictly positive aggregate difference is counted. This
choice is transparent but is not sufficient for a deployment claim. Recommended reporting includes
a threshold curve over absolute score increases, at least
$\xi\in\{0,10^{-3},5\!\times\!10^{-3},10^{-2},5\!\times\!10^{-2}\}$, together with the maximum and
median positive gain. Thresholds should be applied before transfer is computed.

Mutations descending from one base trace are correlated. Confidence intervals should therefore
resample base problems, not individual candidates. A problem-level bootstrap can recompute attack
success, exploit-cell coverage, maximum gain, and transfer for each replicate. With only $120$
problems, maximum-gain intervals may remain wide; reporting a high quantile such as the
$95$th-percentile gain alongside the maximum reduces sensitivity to a single candidate.

\paragraph{Syntactic control}

The five negative-control operators split into two groups: three
\emph{pure-corruption} edits (add a constant to an intermediate number,
drop a reasoning step, swap arithmetic operands) and two \emph{padding}
edits (contentless filler, goal restatement), none adding deceptive
justification. Every control operator, including the padding pair,
\emph{additionally corrupts the final numeric value} (a fixed $+2$ shift
of the last number), because Definition~1 requires a correctness flip:
``contentless'' refers to the inserted text, and the score \emph{rise}
comes from the padding while the correctness flip comes from the
mechanical final-number shift---the pure-corruption operators apply the
same shift without padding and never gain score. In the recorded run
(\texttt{results\_syntactic/}\allowbreak
\texttt{prm\_\allowbreak syntactic\_\allowbreak
control\_\allowbreak summary.json}), the
mean readout yields $47$ strict exploits, \emph{all} from the two padding
operators; the three pure-corruption operators produce none---consistent
with the main-text finding (\S\ref{sec:realprm}) that any locally
innocuous padding, not deception, drives the mean-pooling failure.  This
does not establish that corruption attacks are universally ineffective:
the control set is small and no gradient- or model-guided corruption is
included.

\section{Excluded Reader Failure}
\label{app:readerfailure}

The attempted Skywork-o1-Open-PRM run produced an unparsed native reward channel under the uniform
reader. All extracted per-cell differences were therefore exactly zero. The raw diagnostic row was

\begin{center}
\small
\resizebox{\columnwidth}{!}{%
\begin{tabular}{lrrrr}
\toprule
\textbf{PRM} & \textbf{mean count} & \textbf{min count} & \textbf{coverage} & \textbf{reported gain} \\
\midrule
Skywork-o1-PRM & 0 & 0 & 0.00 & 0.000 \\
\bottomrule
\end{tabular}}
\end{center}

This row is preserved only as a failed-pipeline record. It is excluded from Tables
\ref{tab:multiprm} and \ref{tab:transfer} because a constant reader cannot distinguish robustness
from extraction failure. The corresponding Qwen-to-Skywork and RLHFlow-to-Skywork transfer values
of zero are likewise non-results. A corrected study must implement the model-specific tokenization
and reward extraction path, verify nonconstant outputs on documented examples, and rerun every base
and attacked trace from the beginning.

\section{Archive-Guided Repair Protocol}
\label{app:repairprotocol}

The following protocol turns the theoretical repair objective into a falsifiable experiment without
reusing attack examples for evaluation.  Execution status, stated exactly:
the repair stage, fresh adaptive re-attack, clean-split problem
generalization, MATH-500 replication (\S\ref{sec:realrepair}), the
\emph{paired} pre/post re-attack with the pairwise-margin loss plus
explicit clean anchor (three repair seeds, two attack seeds,
reconciled union statistics, hierarchical bootstrap), the equal-budget
\emph{repair-source} controls (random, exhaustive, strongest-only,
clean-only), the unseen-wording variant-B template arm, and the full
clean-utility panel (\S\ref{sec:pairedrepair}) are \emph{executed}.
The held-out operator family (v1) returned \emph{no repair pairs}
(every repair-half exploit belonged to the held-out family) and
contributes no generalization evidence; an open-ended attack
generator, readout-side defenses, downstream solver evaluation, and
ProcessBench/PRMBench discrimination remain \emph{unexecuted}, are
written in the conditional below, and contribute no number to any
table.

\begin{table}[t]
\centering
\small
\setlength{\tabcolsep}{3.4pt}
\resizebox{\columnwidth}{!}{%
\begin{tabular}{lcccc}
\toprule
\textbf{State} & \textbf{AUROC} & \textbf{Brier} & \textbf{Best-of-4} & \textbf{Margin} \\
\midrule
Pre & .884 & .342 & 1.00 & .192 \\
QD s0/s1/s2 & .935/.915/.921 & .385/.365/.372 & 1.00 & .198--.201 \\
Random s0/s1 & .938/.937 & .392/.390 & 1.00 & .196 \\
Exhaustive s0/s1 & .938/.939 & .392/.391 & 1.00 & .196--.197 \\
Strongest s0/s1 & .941/.949 & .393/.402 & 1.00 & .194--.196 \\
Clean-only & .884 & .342 & 1.00 & .192 \\
\midrule
\multicolumn{5}{l}{\emph{Equal-budget search baselines (original model, eval half,}}\\
\multicolumn{5}{l}{\emph{budget $45$/problem): mean best gain / exploit-cell coverage:}}\\
\multicolumn{5}{l}{\quad MAP-Elites $.021$/$.40$; random $.022$/$.40$;}\\
\multicolumn{5}{l}{\quad single-obj.\ $.009$/$.18$; exhaustive $.024$/$.48$}\\
\bottomrule
\end{tabular}
}
\caption{Top: full clean-utility panel on the evaluation half (ranking
AUROC of correct vs.\ corrupted, Brier of mean-step score against
correctness, mean ranking margin), per repair condition and seed.
Best-of-$4$ is defined exactly: per problem, the clean correct trace
versus three corruption variants (\texttt{confident\_wrong},
\texttt{plausible\_number}, \texttt{authoritative\_frame} at random
magnitude); success iff the clean trace scores strictly highest of the
four---a clean-vs-corrupted ranking ceiling, not robustness in an
attacked candidate pool.  Every adversarial repair improves AUROC and
worsens Brier (score inflation); clean-only changes nothing.  Bottom:
same-statistic equal-budget \emph{search} baselines; the budget exceeds
the $25$-cell grid, so exhaustive enumeration is the strongest
per-problem searcher at this scale.}
\label{tab:cleanutil}
\end{table}

\begin{table}[t]
\centering
\small
\setlength{\tabcolsep}{2.6pt}
\resizebox{\columnwidth}{!}{%
\begin{tabular}{lcccccc}
\toprule
\textbf{Condition} & \textbf{Rate ($k/54$)} & \textbf{Mean gain} & \textbf{Max} & \textbf{q95} & \textbf{$\Delta$gain [CI]} & \textbf{$\Delta$rate} \\
\midrule
Pre (original) & .148 (8) & .0212 & .333 & .146 & --- & --- \\
\midrule
QD, s0 & .037 (2) & .0060 & .194 & .000 & $-.015\,[-.030,-.004]$ & $-6/54$ \\
QD, s1 & .056 (3) & .0068 & .177 & .011 & $-.014\,[-.028,-.004]$ & $-5/54$ \\
QD, s2 & .074 (4) & .0092 & .212 & .034 & $-.012\,[-.023,-.003]$ & $-4/54$ \\
\midrule
Random, s0 & .037 (2) & .0043 & .168 & .000 & $-.017\,[-.033,-.004]$ & $-6/54$ \\
Random, s1 & .037 (2) & .0048 & .172 & .000 & $-.016\,[-.032,-.004]$ & $-6/54$ \\
Exhaustive, s0 & .037 (2) & .0042 & .166 & .000 & $-.017\,[-.034,-.004]$ & $-6/54$ \\
Exhaustive, s1 & .037 (2) & .0048 & .173 & .000 & $-.016\,[-.032,-.004]$ & $-6/54$ \\
Strongest, s0 & .037 (2) & .0053 & .191 & .000 & $-.016\,[-.031,-.004]$ & $-6/54$ \\
Strongest, s1 & .037 (2) & .0048 & .187 & .000 & $-.016\,[-.032,-.004]$ & $-6/54$ \\
Clean-only, s0/s1 & .148 (8) & .0212 & .333 & .146 & $0\,[0,0]$ & $0$ \\
\bottomrule
\end{tabular}
}
\caption{Paired repair v2 on \texttt{Qwen2.5-Math-PRM-7B} (GSM8K,
evaluation half of $54$ problems; identical budgets and denominators
pre/post).  Every statistic, including the rate column, is computed on
the \emph{same} per-problem best-nonnegative-gain vector averaged over
the two attack seeds, so rate columns subtract exactly to $\Delta$rate.
$\Delta$ columns are per-problem paired bootstrap $95\%$ intervals; all
adversarially trained conditions exclude zero.  Repair-source
conditions use an equal pair budget (cap $20$) and equal steps; QD used
$20$ pairs, random/exhaustive/strongest found $13$.  Threshold curves
per condition are in the released JSON
(\texttt{results\_paired\_v2/}\allowbreak
\texttt{prm\_paired\_v2\_summary.json}); the pre-repair curve falls
from $0.148$ at $\tau{=}0$ to $0.056$ at $\tau{=}0.1$.}
\label{tab:paired}
\end{table}

\begin{table}[t]
\centering
\small
\begin{tabular}{llr}
\toprule
\textbf{Source PRM} & \textbf{Target PRM} & \textbf{strict transfer} \\
\midrule
Qwen & RLHFlow & 1.00 \\
RLHFlow & Qwen & 0.93 \\
\bottomrule
\end{tabular}
\caption{Cross-PRM transfer for mean-readout exploits. Transfer records preservation of a strict
positive score change, not preservation of the source effect size. Thresholded transfer curves and
problem-level confidence intervals are required before deployment-level conclusions.}
\label{tab:transfer}
\end{table}

\begin{table}[t]
\centering
\small
\resizebox{\columnwidth}{!}{%
\begin{tabular}{lrrrrr}
\toprule
\textbf{PRM} & \textbf{evals/readout} & \textbf{mean} & \textbf{min} & \textbf{mean cov.} & \textbf{$g_{\max}$} \\
\midrule
\makecell[l]{Qwen2.5-Math-\\PRM-7B} & 5400 & 44  & 1   & 0.40 & 0.294 \\
\makecell[l]{RLHFlow-Llama3.1-\\PRM} & 5400 & 645 & 675 & 1.00 & 0.005 \\
\bottomrule
\end{tabular}%
}
\caption{Larger discovery audit over $120$ sampled base traces. The
evals/readout column is the \emph{allocated} budget
($120\times45=5{,}400$); only the $107$ traces whose base answer passes
exact-match verification are attacked, consuming $107\times45=4{,}815$
evaluations, and all counts, coverage, and problem-level statistics use
$n=107$ as denominator. ``mean'' and ``min'' are candidate-level
strict exploit counts. The coverage and maximum-gain columns summarize the mean-readout archive,
matching the archive used for the transfer audit. Counts should be interpreted together with
$g_{\max}$: a strict positive event can be numerically small.}
\label{tab:multiprm}
\end{table}

\paragraph{Problem and attack splits}

Problems should be divided into disjoint archive-training, calibration, and final-test partitions.
Operator surface forms should be split independently so that held-out paraphrases instantiate known
families, while at least one entire operator family is reserved for an unseen-mechanism test. The QD
search runs only on the archive-training problems when constructing $A_+$. Hyperparameters,
including the margin $m$, repair weight $\lambda$, and operational threshold $\xi$, are selected on
the calibration problems.

\paragraph{Repair and clean anchoring}

The PRM is fine-tuned with the pairwise margin objective from Section~\ref{sec:certificate}. Each
batch should mix clean process examples, ordinary incorrect traces, and archived exploits. A
frozen-copy regularizer can penalize movement of clean step probabilities, while the original clean
objective preserves error-detection ability. Merely lowering every attacked score is insufficient:
the repair must retain or improve the separation of correct and naturally incorrect traces.

\paragraph{Fresh adaptive re-audit}

After training, MAP-Elites is restarted from independent seeds on final-test problems. The re-audit
includes known operators with held-out wording, the reserved operator family, and an open-ended
LLM mutation channel. The attack budget and PRM-query count are fixed before comparison. Cells are
added to $\R$ only from this post-training re-audit; training-set attack fit is reported separately.
To measure harmful spillover, the same search is run in cells outside $\R$ before and after repair,
producing an empirical estimate of $\zeta_{\R}$.

\paragraph{Baselines and outcomes}

Archive-guided repair should be compared under equal training-example and PRM-query budgets with
random attack augmentation, training on globally strongest attacks only, exhaustive evaluation of
the discrete $25$-cell grid, and a no-repair control. Readout baselines should include mean,
minimum, product, last-step, a lower quantile, and a differentiable soft minimum. Discovery is
reported through attack success, exploit-cell coverage, QD score, maximum and high-quantile gain,
and semantic uniqueness. Repair is reported on held-out and adaptive attacks. Clean utility is
reported through step-error detection, correct-versus-incorrect ranking, calibration, and
best-of-$n$ answer selection. All intervals should resample problems, and all model-reader failures
should be excluded before aggregation.

This design separates four outcomes that a single exploit count cannot distinguish: memorization of
archived strings, generalization within a descriptor cell, transfer to an unseen attack family, and
preservation of clean PRM utility. Only the latter three support a substantive defense claim.

\section{Supplementary Figures}
\label{sec:supp_figures}

We next evaluate the discovered vulnerabilities and the effectiveness of the
repair procedure from several complementary perspectives. Beyond measuring
whether exploits can be found, we analyze their prevalence, severity,
threshold sensitivity, generalization across reasoning benchmarks, and impact
on clean-task behavior. Since exploit frequency and exploit magnitude capture
different failure modes, we report both problem-level prevalence and reward
gain statistics throughout the evaluation.

The repair procedure reduces both attack prevalence and attack severity.
Figure~\ref{fig:repair_effect_all_conditions}(a) demonstrates a substantial
reduction in exploit frequency across repair conditions, while
Figure~\ref{fig:repair_effect_all_conditions}(b) shows a corresponding
suppression of high-gain attack tails. Residual maximum gains after fresh
adaptive re-attack indicate that repair reduces, but does not completely
eliminate, worst-case vulnerability.

Aggregation choice and PRM architecture also affect the observed exploit
landscape. Figure~\ref{fig:aggregation_sensitivity}(a) shows that mean and
minimum aggregation produce substantially different strict exploit counts for
Qwen2.5-Math-PRM-7B and RLHFlow-Llama3.1-PRM. However, exploit count alone does
not determine practical risk: Figure~\ref{fig:aggregation_sensitivity}(b)
shows that the two PRMs exhibit markedly different maximum reward gains,
highlighting that vulnerability depends on both aggregation strategy and the
underlying reward model.

To distinguish strict detections from practically meaningful attacks, we
perform a threshold sensitivity analysis. As shown in
Figure~\ref{fig:threshold_curve}, the exploit rate decreases monotonically
from $0.168$ at $\tau=0$ to $0.028$ at $\tau=0.2$, demonstrating that
zero-threshold exploit counts include a substantial number of low-margin
events.

We further examine whether the discovered phenomena generalize beyond GSM8K.
Figure~\ref{fig:math500_generalization}(a) shows that mean readout identifies
41 strict exploits on MATH-500, while minimum pooling reduces this count to 19
and post-repair evaluation eliminates all detected exploits. Figure
~\ref{fig:math500_generalization}(b) shows that minimum pooling primarily
reduces severity, decreasing the worst-case cell gap from $0.232$ to $0.020$,
whereas repair removes the observed nonnegative-gain exploits.

Finally, we evaluate whether repair preserves clean-task utility.
Figure~\ref{fig:clean_utility}(a) shows improved ranking AUROC across repair
conditions compared with the original model, while
Figure~\ref{fig:clean_utility}(b) reveals the corresponding increase in Brier
score, consistent with score inflation. Despite this calibration shift,
Figure~\ref{fig:clean_utility}(c) confirms that Best-of-4 clean ranking remains
perfect at $1.00$ across all conditions, indicating no degradation in the
evaluated clean ranking behavior.

\begin{figure*}[t]
    \centering
    \includegraphics[width=\textwidth]{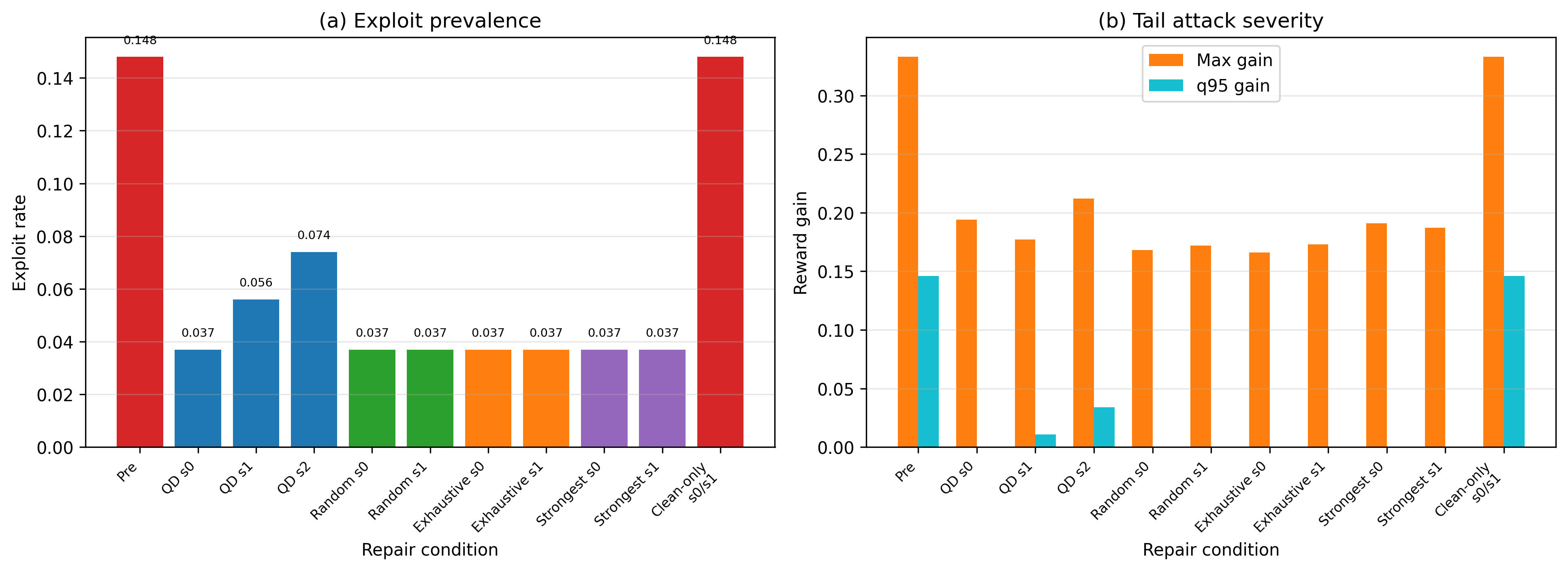}
    \caption{
   Paired repair analysis across repair conditions on the GSM8K evaluation set. (a) Exploit prevalence, measured as the fraction of problems exhibiting positive-gain exploits. (b) Attack severity, measured by the maximum observed reward gain and the 95th-percentile reward gain. Repair substantially reduces exploit prevalence and tail severity, but residual maximum gains remain after fresh adaptive re-attack.
    }
    \label{fig:repair_effect_all_conditions}
\end{figure*}

\begin{figure*}[t]
    \centering
    \includegraphics[width=\textwidth]{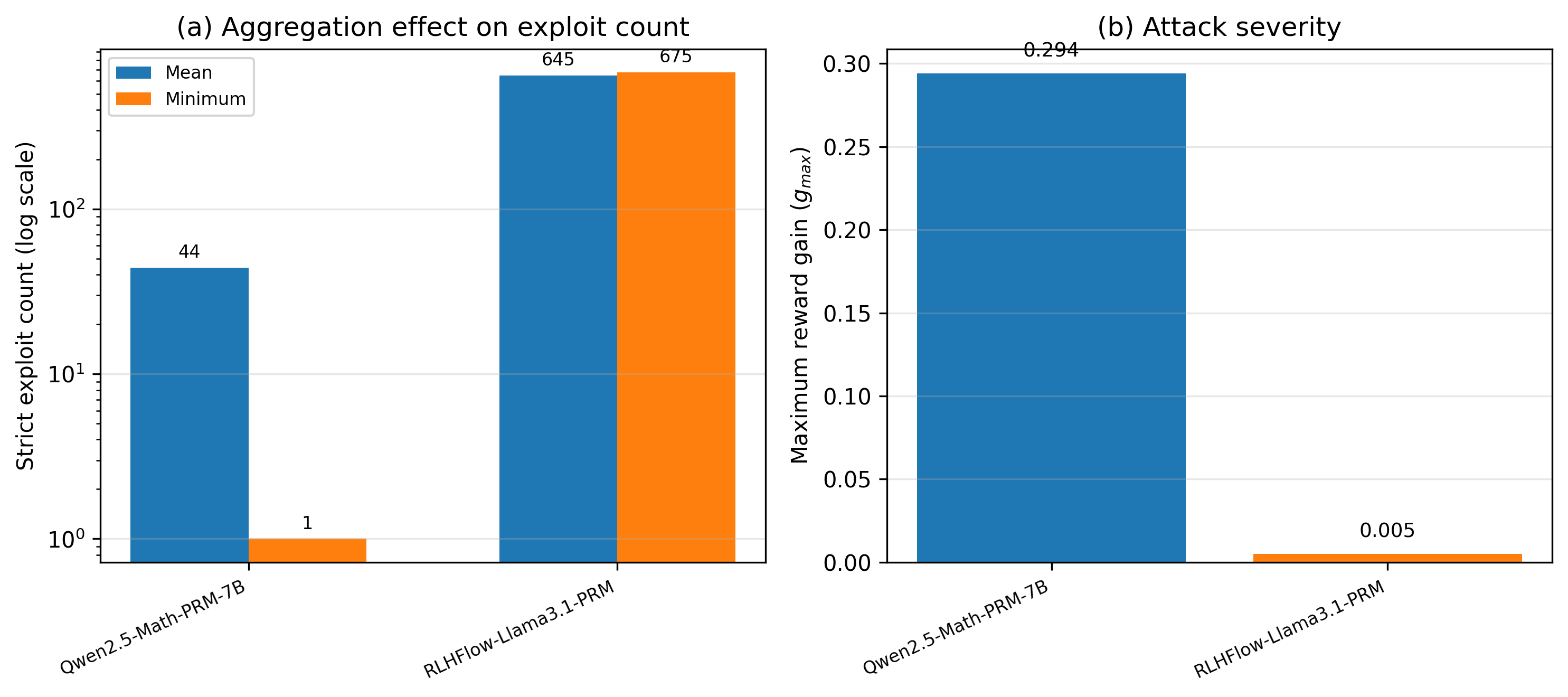}
\caption{Aggregation and PRM dependence of exploit discovery. 
(a) Candidate-level strict exploit counts under mean and minimum readouts for 
Qwen2.5-Math-PRM-7B and RLHFlow-Llama3.1-PRM. 
(b) Maximum gain ($g_{\max}$) from the mean-readout archive, showing 
substantially different attack severity across the two PRMs. 
The results show that exploit counts and severity vary with both the 
aggregation rule and the PRM.}
    \label{fig:aggregation_sensitivity}
\end{figure*}

\begin{figure*}[t]
    \centering
    \includegraphics[width=0.65\linewidth]{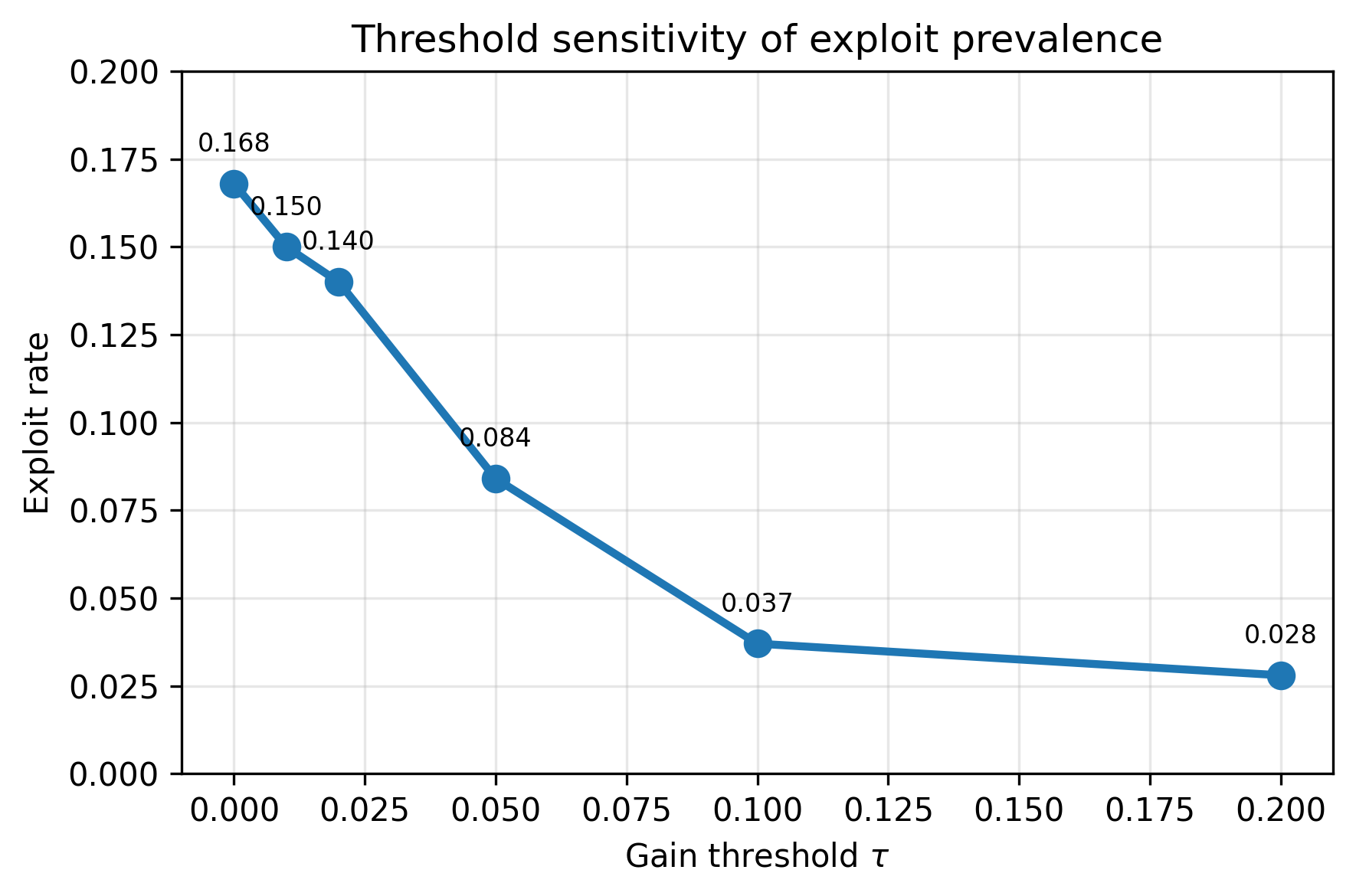}
  \caption{Exploit rate as a function of the gain threshold ($\tau$) on GSM8K. Increasing $\tau$ reduces the exploit rate, showing that strict zero-threshold counts include many low-margin events.}
    \label{fig:threshold_curve}
\end{figure*}

\begin{figure*}[t]
    \centering
    \includegraphics[width=\linewidth]{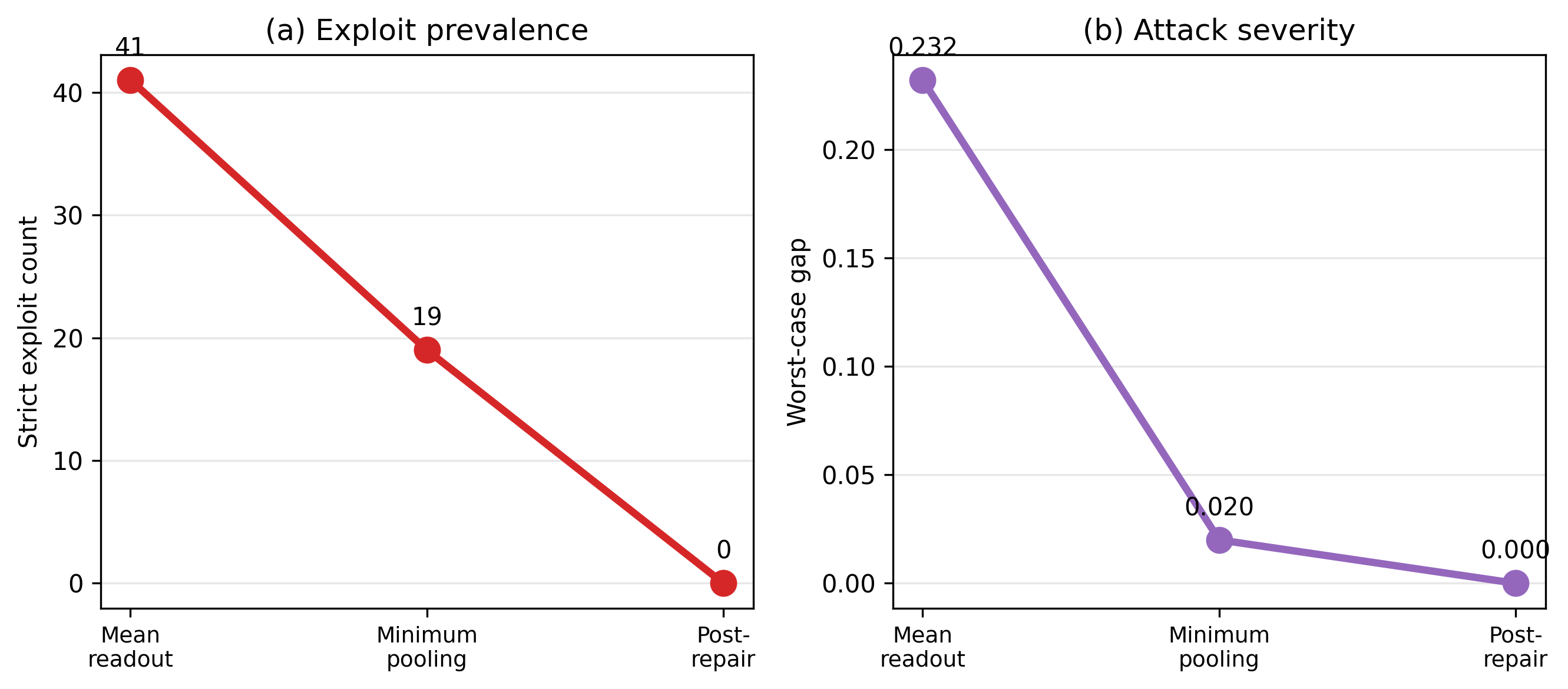}
    \caption{Generalization of exploit discovery and repair on MATH-500. (a) Strict exploit counts decrease from mean to minimum pooling and fall to zero after repair. (b) Maximum cell-gap severity decreases substantially under minimum pooling and reaches zero under the nonnegative-gain convention after repair. The results indicate that minimum pooling primarily suppresses exploit severity, whereas repair eliminates the observed nonnegative-gain exploits.}
    \label{fig:math500_generalization}
\end{figure*}

\begin{figure*}[t]
    \centering
    \includegraphics[width=\linewidth]{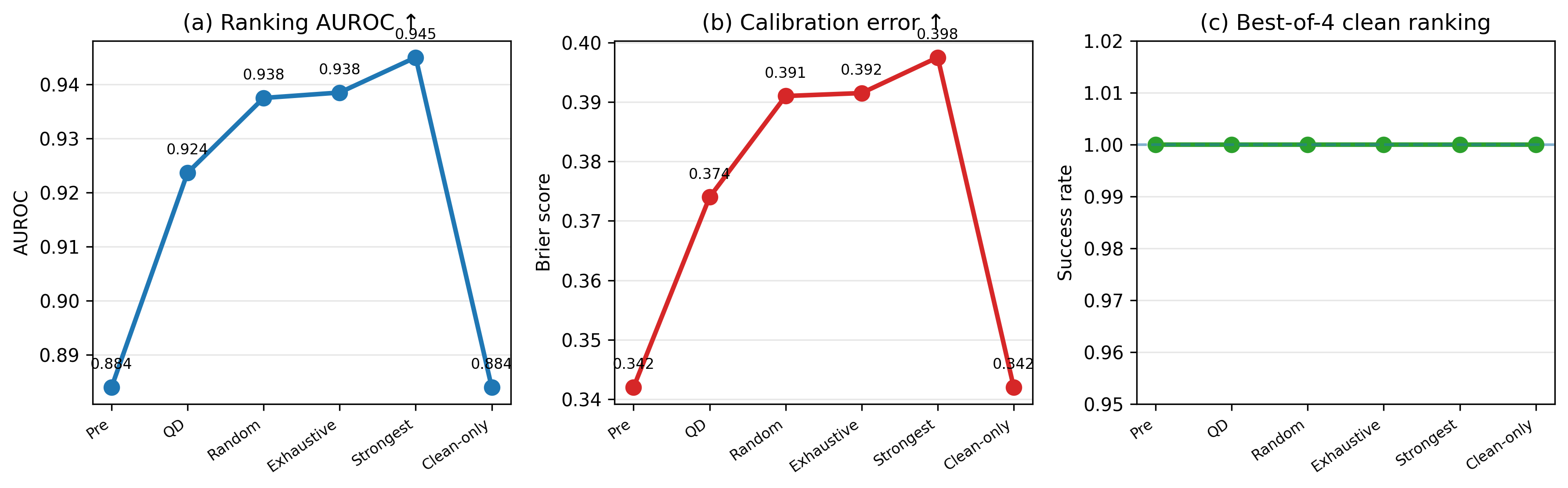}
    \caption{Clean-task ranking and scoring evaluation after repair. (a) Ranking AUROC improves across adversarial repair conditions relative to the original model, indicating improved discrimination between correct and corrupted traces. (b) Brier scores increase after repair, reflecting score inflation despite improved ranking quality. (c) Best-of-4 clean-vs-corruption ranking remains perfect across all conditions, showing no degradation on this evaluation.}
    \label{fig:clean_utility}
\end{figure*}
\end{document}